\documentclass[12pt, final]{l4dc2022}

\newtheorem{assumption}{Assumption}

\title[Gaussian Processes and Variational Integrators]{Structure-Preserving Learning Using Gaussian Processes and Variational Integrators}

\usepackage{times}
\usepackage[printonlyused,withpage]{acronym}
\usepackage{tikz}
\usetikzlibrary{calc}

\newcommand{\bs}[1]{\boldsymbol{#1}}
\newcommand{\hT}{^\text{T}}

\author{
 \Name{Jan Br\"udigam} \Email{jan.bruedigam@tum.de}
 \AND
 \Name{Martin Schuck} \Email{martin.schuck@tum.de}
 \AND
 \Name{Alexandre Capone} \Email{alexandre.capone@tum.de}
 \AND
 \Name{Stefan Sosnowski} \Email{sosnowski@tum.de}
 \AND
 \Name{Sandra Hirche} \Email{hirche@tum.de}\\
 \addr Chair of Information-oriented Control\\
 Technical University of Munich, Germany
}

\begin{document}

\maketitle

\begin{abstract}
    Gaussian process regression is increasingly applied for learning unknown dynamical systems. In particular, the implicit quantification of the uncertainty of the learned model makes it a promising approach for safety-critical applications. When using Gaussian process regression to learn unknown systems, a commonly considered approach consists of learning the residual dynamics after applying some generic discretization technique, which might however disregard properties of the underlying physical system. Variational integrators are a less common yet promising approach to discretization, as they retain physical properties of the underlying system, such as energy conservation and satisfaction of explicit kinematic constraints. In this work, we present a novel structure-preserving learning-based modelling approach that combines a variational integrator for the nominal dynamics of a mechanical system and learning residual dynamics with Gaussian process regression. We extend our approach to systems with known kinematic constraints and provide formal bounds on the prediction uncertainty. The simulative evaluation of the proposed method shows desirable energy conservation properties in accordance with general theoretical results and demonstrates exact constraint satisfaction for constrained dynamical systems.
\end{abstract}


\section{Introduction}

Control of complex mechanical systems such as robots typically requires good knowledge of the system dynamics to achieve high control performance. In this regard, learning-based approaches are promising, as they can accurately model additional influences on top of the nominal dynamics, such as friction or external forces. One promising learning-based method for such scenarios is Gaussian process regression, which yields good generalization from only few training samples \citep{rasmussen_gaussian_2006}. In addition, uncertainty bounds for the learned dynamics can be given \citep{lederer2019uniform,srinivas_information_2012}, rendering the regression useful for safety-critical control applications, as shown in, for example, \cite{umlauft_learning_2017} and \cite{capone2019backstepping}. 

While much effort has been devoted to the theoretical properties of GP regression, the appropriate discretization of continuous-time mechanical systems in combination with GPs for simulation and control has found little attention. \cite{marsden_discrete_2001} proposed variational (symplectic) integrators as generally suitable discretization methods for mechanical systems, as they retain certain structural properties, such as energy or momentum conservation, and constraint satisfaction. As a result, they achieve higher accuracy than generic integrators such as explicit Runge-Kutta methods, which allows for larger prediction time steps---an important feature in, for example, online control. However, no attempts have been made to exploit the advantageous properties of variational integrators in combination with Gaussian process regression.

The contribution of this paper is a novel approach for physical-structure-preserving and data-driven modelling of dynamical systems. To this end, we combine a variational integrator of the nominal model dynamics and Gaussian process regression to include initially unknown system dynamics. In this way, we retain physical properties of the known system dynamics and augment them by the learned residual dynamics. This concept is also extended to learning system dynamics with known explicit kinematic constraints, for which variational integrators are particularly well suited. In addition, formal uncertainty bounds for learning unconstrained and constrained dynamical systems in combination with variational integrators are provided. The energy conservation and constraint satisfaction properties of the proposed approach are illustrated in simulations. Furthermore, the performance of our method is numerically studied for several mechanical systems including such with kinematic constraints and in different coordinates.

The rest of the paper is structured as follows: in Sec. \ref{sec:background} key concepts of Gaussian process regression and variational integrators are summarized, and related work is listed. Section \ref{sec:main} includes a derivation of the integrators for unconstrained and constrained dynamical systems in combination with Gaussian process regression, which are then evaluated and discussed in Sec. \ref{sec:evaluation}. Finally, we present our conclusions in Sec. \ref{sec:conclusions}.

\section{Background}\label{sec:background}
We briefly revisit the relevant concepts of Gaussian process regression and describe a first-order variational integrator. A more thorough treatment of Gaussian processes and variational integrators can be found in \cite{rasmussen_gaussian_2006} and \cite{marsden_discrete_2001}, respectively. The section is concluded with a review of selected related literature.

\subsection{Gaussian Process Regression}
A Gaussian process (GP) is employed to model an unknown function $f(\cdot)$. A GP is a collection of random variables, of which any finite subset is jointly normally distributed. As such, GPs generalize the notion of Gaussian distributions to infinite-dimensional vectors. A GP, denoted $f(\cdot) \sim \mathcal{GP}(m(\cdot), k(\cdot,\cdot))$, is specified by its mean function $m(\cdot)$ and covariance function $k(\cdot,\cdot)$. 

By conditioning a GP on measurement data pairs $(\bs{z}_i,f(\bs{z}_i))_{i=1...N} =: (\bs{Z},\bs{y})$, with states $\bs{Z} = (\bs{z}_i)_{i=1...N}$ and measurements $\bs{y} = (f(\bs{z}_i))_{i=1...N}$,  an unknown function $f(\cdot)$ can be modeled. The resulting posterior mean and variance are then given by 
\begin{subequations}
	\begin{align}
		\label{eq:gpmean}
		\mu(\bs{z}) &= \bar{m}(\bs{z}) +  \bs{k}(\bs{z},\bs{Z})\bs{K}(\bs{Z},\bs{Z})^{-1}\left(\bs{y}-\bar{\bs{m}}(\bs{Z})\right),\\
		\label{eq:gpvar}
		\sigma^2(\bs{z}) &= k(\bs{z},\bs{z}) - \bs{k}(\bs{z},\bs{Z})\bs{K}(\bs{Z},\bs{Z})^{-1}\bs{k}(\bs{z},\bs{Z})\hT,
	\end{align}
\end{subequations}
where $\bar{m}(\cdot)$ is the prior mean function, $\bs{k}(\bs{z},\bs{Z})=[k(\bs{z},\bs{z}_1) ~\cdots~ k(\bs{z},\bs{z}_N)]$, and the entries $k_{ij}$ of the posterior covariance matrix $\bs{K}$ are given by $k_{ij}(\bs{Z},\bs{Z}) = k(\bs{z}_i,\bs{z}_j)$.

While there exist different methods for predicting multidimensional function values, in this paper we model each entry of a vector-valued function with a separate GP. This is a standard procedure in GP-based literature \citep{deisenroth2013gaussian}. Therefore, in the following, a vector-valued GP $\bs{f}(\cdot) \sim \mathcal{GP}(\bs{m}(\cdot), \bs{k}(\cdot, \cdot))$ corresponds to a vector of independent GPs.

\subsection{Variational Integrators}
Solving differential equations, for example the equations of motion of a dynamical system, requires numerical integration. Variational integrators, as described in \cite{marsden_discrete_2001}, are symplectic (``shape-preserving'') integrators for mechanical systems that are able to preserve certain physical properties of the underlying system, such as energy conservation or satisfaction of explicit constraints without drift. 

Given a continuous-time mechanical system with Lagrangian $\mathcal{L}(\bs{x}(t),\bs{v}(t))$, where $\bs{x}\in\mathbb{R}^n$ and $\bs{v}\in\mathbb{R}^n$ are the generalized position and velocity of the system, respectively, as well as kinematic constraints $\bs{g}(\bs{x}(t))=\bs{0}\in\mathbb{R}^c$, the action integral $S$ of the system can be written as
\begin{equation}\label{eq:action}
    S = \int^{t_N}_{t_0}  \mathcal{L}\left(\bs{x}(t),\bs{v}(t)\right)~\mathrm{d}t + \int^{t_N}_{t_0}\bs{\lambda}(t)\hT\bs{g}(\bs{x}(t))~\mathrm{d}t,
\end{equation}
where $\bs{\lambda}\in\mathbb{R}^c$ is a Lagrange multiplier (constraint force) enforcing the constraints $\bs{g}=\bs{0}$.

According to the principle of least action, minimizing \eqref{eq:action} by varying the trajectory $\bs{x}(t)$ yields the continuous-time differential equations of the system that could then be discretized. However, to obtain a variational integrator instead, \eqref{eq:action} is discretized directly. For clarity, we will derive and use a first-order integrator throughout this paper, but higher-order discretizations are also possible.

A first-order variational integrator is obtained by discretizing \eqref{eq:action} over three time steps:
\begin{equation}\label{eq:discrete_action}
	S_\mathrm{d} = \sum^{1}_{k=0} \left(\mathcal{L}(\bs{x}_k,\bs{v}_k) + \bs{\lambda}_k\hT\bs{g}(\bs{x}_k)\right)\Delta t,
\end{equation}
with time step $\Delta t$ and where
\begin{equation}\label{eq:discrete_vel}
    \bs{v}_k = \frac{\bs{x}_{k+1}-\bs{x}_k}{\Delta t}.
\end{equation}

Least action for fixed start and end points $\bs{x}_0$ and $\bs{x}_2$, i.e., minimizing \eqref{eq:discrete_action} with respect to the center point $\bs{x}_1$, yields the implicit discretized equations of motion
\begin{equation}\label{eq:dynamics}
	\nabla_{\bs{x}_1} S_\mathrm{d} =: -\bs{d} = \bs{0}.
\end{equation} 
The resulting implicit nonlinear equations of motion (variational integrator) take the form
\begin{subequations}\label{eq:var_int}
	\begin{align}
		\bs{x}_{k+1} &= \bs{x}_k + \bs{v}_k\Delta t,\label{eq:dynamics_update}\\
		\bs{d}(\bs{v}_{k+1},\bs{\lambda}_{k+1}) &= \bs{d}_0(\bs{v}_{k+1}) - \bs{G}(\bs{x}_{k+1})\hT\bs{\lambda}_{k+1} = \bs{0},\label{eq:dynamics_implicit}\\
		\bs{g}(\bs{x}_{k+2}) &= \bs{g}(\bs{x}_{k+1},\bs{v}_{k+1}) = \bs{0},\label{eq:dynamics_constraint}
	\end{align}
\end{subequations}
where $\bs{d}_0$ are the unconstrained dynamics, and the constraint Jacobian $\bs{G}(\bs{x}) = \frac{\partial \bs{g}(\bs{x})}{\partial \bs{x}}$ maps constraint forces into the dynamics. Note that the constraints are enforced for $\bs{x}_{k+2}$.

Besides the typical parameterization of mechanical systems in minimal coordinates (joint coordinates), so-called maximal coordinates can be used as well. In this case, each body of a system is described with its six degrees of freedom and kinematic constraints represent joints connecting the bodies. The general idea for the derivation of the variational integrator remains the same, and details including the treatment of quaternion-based orientation representations are given in \cite{brudigam_linear-time_2020}. Our method is also applicable to maximal-coordinate descriptions which have demonstrated numerical and control-theoretical advantages, as described in \cite{brudigam_linear-quadratic_2021} and \cite{brudigam_linear-time_2021-1}.

\subsection{Related Work}
Theoretical results for variational integrators are presented in \cite{marsden_discrete_2001}, and numerous extensions can be found in subsequent works, for example in \cite{junge_discrete_2005} and \cite{wenger_construction_2017}. Additionally, there exist efficient algorithms for variational integrators in minimal coordinates \citep{lee_linear-time_2018,fan_efficient_2019} and maximal coordinates \citep{brudigam_linear-time_2020}. With respect to learning-based approaches, variational integrators have found increasing attention in the works of \cite{Saemundsson_variational_2020}, \cite{desai_variational_2021}, and \cite{zhong_unsupervised_2020} in combination with neural networks, but not with Gaussian process regression.

Structured learning of system dynamics with Gaussian process regression is gaining increasing attention. A common approach is to include as much information as possible about the system in the prior mean, making the residual smaller and therefore easier to learn, for example as shown in \cite{koller2018learning} and \cite{capone2019backstepping}. Alternatively, several methods exist for constructing kernel functions of a Gaussian process such that they generate predictions obeying certain physical constraints or properties. This principle is employed in \cite{cheng_learn_2016} to generate kernels spanning a subspace that captures the Lagrangian's projection as inverse dynamics. In a similar spirit, \cite{umlauft2019feedback} propose a composite kernel that captures the control-affine structure of systems where feedback linearization can be applied. However, these approaches make use of standard discretization techniques, and there exist no approaches utilizing the advantages of variational integrators. An alternative line of research considers the inclusion of constraints in GPs. \cite{swiler_survey_2020} provide a survey of approaches for additionally treating constraints with Gaussian process regression. Amongst others, boundary constraints can be satisfied with so-called warping, where values are transformed to and from the constrained set with a monotone warping function, as described in \cite{snelson_warped_2004}. Constraints in the form of (partial) differential equations can be incorporated in Gaussian process regression as well, for example as shown in \cite{raissi_machine_2017} and \cite{owhadi_bayesian_2015}. \cite{geist_learning_2020} incorporate knowledge of the constraints of a continuous-time system to linearly transform a Gaussian process modeling the unconstrained accelerations, such that the accelerations adhere to constraints linear in the accelerations. A disadvantage of these methods is that they require the system at hand to satisfy strict requirements or are formulated in continuous time which might not apply in practice. In contrast, our method is formulated directly in discrete time and can handle nonlinear kinematic constraints. \cite{ensinger_structure_2021} investigate structure-preserving integration for Gaussian process regression, but they do not directly integrate the GP into the integrator.

\section{Combining Variational Integrators and Gaussian Process Regression}\label{sec:main}
Our goal is to predict the next state $\bs{z}_{k+1} = [\bs{x}_{k+1}\hT ~~ \bs{v}_{k+1}\hT]\hT$ of a dynamical system given the current state $\bs{z}_{k} = [\bs{x}_{k}\hT ~~ \bs{v}_{k}\hT]\hT$, the nominal dynamics model \eqref{eq:var_int}, and training data $(\bs{Z},\bs{y}) = (\bs{z}_{i},\bs{v}_{i+1})_{i=1...N}$ obtained from trajectories with random initial configurations. In the following, we first present an approach for predicting the next state for unconstrained and constrained systems. Afterwards, formal uncertainty bounds are stated, which is nontrivial for constrained systems due to the projected Gaussian process prediction.

\subsection{One-Step Prediction}
The next position $\bs{x}_{k+1}$ is directly obtained from the variational integrator (cf. \eqref{eq:dynamics_update}) without requiring any dynamics information:
\begin{equation}
    \bs{x}_{k+1} = \bs{x}_k + \bs{v}_k\Delta t.
\end{equation}
Therefore, the GP is only required for computing the next velocity $\bs{v}_{k+1}$. Hence, we use a posterior GP mean to model the one-step predictive model, i.e.,
\begin{equation}
    \bs{v}_{k+1}(\bs{z}_k) := \mu(\bs{z}_k),
\end{equation}
where $\mu(\bs{z}_k)$ is computed as in \eqref{eq:gpmean}.

\subsubsection{Unconstrained Systems}
After training the GP with the training data, a posterior mean function for the unconstrained velocity $\bs{v}_{k+1}(\bs{z}_k)$ is obtained and the integrator for the unconstrained dynamics is constructed as
\begin{subequations}\label{eq:var_int_gp}
	\begin{align}
		\bs{x}_{k+1}(\bs{z}_k) &= \bs{x}_k + \bs{v}_{k}\Delta t,\label{eq:dynamics_update_gp}\\
        \bs{0} &= \bs{d}_0(\bar{\bs{v}}_{k+1}),\label{eq:dynamics_implicit_gp1}\\
		\bs{v}_{k+1}(\bs{z}_k) &= \bar{\bs{v}}_{k+1} + \bs{k}(\bs{z}_k,\bs{Z})K(\bs{Z},\bs{Z})^{-1}\left(\bs{y}-\bar{\bs{v}}_{k+1}(\bs{Z})\right),\label{eq:dynamics_implicit_gp2}
	\end{align}
\end{subequations}
where the prior mean $\bar{\bs{v}}_{k+1}$ is obtained by solving the (implicit) nominal dynamics \eqref{eq:dynamics_implicit_gp1} resulting from the variational integrator \eqref{eq:var_int}, for example with Newton's method.

With integrator \eqref{eq:var_int_gp}, unmodeled and potentially non-conservative dynamics, such as friction, can be described by the learned regression model, and the desirable properties of the variational integrator, such as energy conservation, are retained for the nominal and correctly learned dynamics. So, given an underlying Lagrangian system and assuming for all $\bs{z}_k$ the velocity $\bs{v}_{k+1}(\bs{z}_k)$ obtained from the integrator \eqref{eq:var_int_gp} matches the velocity obtained from the variational integrator \eqref{eq:var_int} directly derived for the underlying system, i.e., both produce the same next state, then the integrators \eqref{eq:var_int_gp} and \eqref{eq:var_int} are equivalent variational (symplectic) integrators with corresponding theoretical properties of variational integrators. Deviations of the prediction \eqref{eq:dynamics_implicit_gp2} from the system's true velocity no longer allow for a direct statement regarding the equivalence of the integrator \eqref{eq:var_int_gp} and variational integrator \eqref{eq:var_int}.
Nonetheless, we obtained satisfying results in practice, for example regarding conservation of energy, as demonstrated in Sec. \ref{sec:evaluation}.

\subsubsection{Constrained Systems}
Now, systems with kinematic constraints $\bs{g}(\bs{x})=\bs{0}$ are considered. We assume to know the constraints correctly in advance, i.e., the kinematic constraints need not be learned. Such constraints can be obtained from measurements or are prior knowledge.

The general approach for kinematically constrained systems remains the same as before, but now an additional projection of the predicted state onto the constraint manifold is required to satisfy the constraints. As before, a velocity $\bs{v}_{\mathrm{u},k+1}$ is calculated with \eqref{eq:dynamics_implicit_gp2} from the modeled dynamics and the Gaussian process regression. However, this velocity will not satisfy $\bs{g}(\bs{x}_{k+1},\bs{v}_{\mathrm{u},k+1}) = \bs{0}$ in general, since the constraints are not explicitly incorporated in the Gaussian process regression. Note that the regression model was obtained from data from the actual constrained system and therefore the regression prediction will still be close to the constraint-satisfying velocity.  

The constraint-fulfilling velocity $\bs{v}_{k+1}$ is calculated with a nonlinear constrained least-squares optimization, resulting in the following integrator for constrained dynamics:
\begin{subequations}\label{eq:var_int_gp_con}
	\begin{align}
		\bs{x}_{k+1}(\bs{z}_k) &= \bs{x}_k + \bs{v}_{k}\Delta t,\label{eq:dynamics_update_gp_con}\\
        \bs{0} &= \bs{d}_0(\bar{\bs{v}}_{k+1}) - \bs{G}(\bs{x}_{k+1})\hT\bs{\lambda}_{k+1},\label{eq:dynamics_implicit_gp_con1}\\
        \bs{0} &= \bs{g}(\bs{x}_{k+1},\bar{\bs{v}}_{k+1}),\label{eq:dynamics_constraint_gp_con}\\
		\bs{v}_{\mathrm{u},k+1}(\bs{z}_k) &= \bar{\bs{v}}_{k+1} + \bs{k}(\bs{z}_k,\bs{Z})K(\bs{Z},\bs{Z})^{-1}\left(\bs{y}-\bar{\bs{v}}_{k+1}(\bs{Z})\right),\label{eq:dynamics_implicit_gp_con2}\\
        \bs{v}_{k+1} &= ~\underset{\bs{v}_{k+1}}{\mathrm{argmin}} ~  \lVert\bs{v}_{k+1}-\bs{v}_{\mathrm{u},k+1}\rVert^2, ~~ \text{s.t.} ~~ \bs{g}(\bs{x}_{k+1},\bs{v}_{k+1}) = \bs{0}.\label{eq:dynamics_projection_gp_con}
	\end{align}
\end{subequations}
The integrator \eqref{eq:var_int_gp_con} consists of the variational integrator of the nominal model \eqref{eq:dynamics_update_gp_con} - \eqref{eq:dynamics_constraint_gp_con}, the prediction from the GP regression \eqref{eq:dynamics_implicit_gp_con2}, and the projection of the predicted velocity onto the constraint manifold \eqref{eq:dynamics_projection_gp_con}. Due to the projection, we recover exact constraint satisfaction also for the learned model. The optimization problem \eqref{eq:dynamics_projection_gp_con} can be solved with a variety of approaches, for example with numerical methods described in \cite{nocedal_numerical_2006}. 

As before, if the predicted velocity $\bs{v}_{\mathrm{u},k+1}$ matches the velocity obtained from the variational integrator \eqref{eq:var_int} for the underlying system, we obtain equivalent variational integrators. Otherwise, no direct statement regarding the equivalence can be made. Due to the projection \eqref{eq:dynamics_projection_gp_con}, technically the resulting velocity $\bs{v}_{k+1}$ no longer follows a jointly Gaussian distribution and therefore is not a GP anymore. However, we recover a probabilistic error bound for the output after projection in the next section and, following \cite{swiler_survey_2020}, we refer to the overall regression as GP regression.

\subsection{Prediction Error Bound}
One of the major advantages of GP regression is the availability of an error bound for the posterior mean, which comes in the form of the posterior variance multiplied by a scalar \citep{chowdhury2017kernelized,lederer2019uniform,capone2021gaussian,srinivas_information_2012}.

In general, the projection \eqref{eq:dynamics_projection_gp_con} is a nonlinear optimization problem which is not monotone and, hence, its inverse does not necessarily exist. Therefore, warping approaches maintaining a measure of the variance, such as those described by \cite{snelson_warped_2004}, cannot be applied to obtain uncertainty bounds on the velocity $\bs{v}_{k+1}$. However, a formal upper bound for the uncertainty can be given under the following assumption.
\begin{assumption}\label{assu:lipschitz}
    The projection \eqref{eq:dynamics_projection_gp_con} corresponds to a Lipschitz continuous function with Lipschitz constant $L$.
\end{assumption}
Assumption \ref{assu:lipschitz} is violated if a predicted velocity lies close to multiple possible constraint-satisfying velocities, in which case the argmin value is discontinuous. Note that the same issue would occur in the nominal variational integrator if the initial guess for solving the implicit dynamics is far off. That is to say, for reasonable predictions of the GP, Assumption \ref{assu:lipschitz} will be fulfilled, just as the nominal variational integrator finds the correct velocity for a reasonable initial guess. In practice the Lipschitz constant $L$ can be estimated by sampling the gradient of the function \eqref{eq:dynamics_projection_gp_con} at several points in the interval of interest.

\begin{theorem}\label{theo:bounds}
    Let Assumption \ref{assu:lipschitz} hold. For all $\bs{z}$, let $\bs{v}_{\mathrm{u}}(\bs{z}) \sim \mathcal{GP}(\bar{\bs{v}}_{\mathrm{u}}(\bs{z}), \sigma^2_{\mathrm{u}}(\bs{z}))$ denote a sample from a GP with posterior mean $\bar{\bs{v}}_{\mathrm{u}}(\bs{z})$ and variance $\sigma^2_{\mathrm{u}}(\bs{z})$, and consider a least-squares projection with nonlinear constraints $\bs{g}=\bs{0}$ of the form
    \begin{equation*}
        \bs{v}(\bs{v}_{\mathrm{u}}) = ~\underset{\bs{\nu}}{\mathrm{argmin}} ~  \lVert\bs{\nu}-\bs{v}_{\mathrm{u}}\rVert^2, ~~ \mathrm{s.t.} ~~ \bs{g}(\bs{\nu}) = \bs{0}.
    \end{equation*}
	Then for every $\delta \in (0,1)$, there exists a $\gamma>0$, such that
    \[\lVert \bs{v}(\bs{v}_{\mathrm{u}}) - \bs{v}(\bar{\bs{v}}_{\mathrm{u}})\rVert < \sum_{i=1}^{d_v}{\gamma}\sigma_{u,i}(\bs{z})\]
    holds for all states $\bs{z}$ with a probability at least $1-\delta$, where $d_v$ is the dimension of $\bs{v}_u$ and $\sigma_{u,i}(\cdot)$ is the posterior variance of the $i$-th entry of the Gaussian process used to model $\bs{v}_u$.
\end{theorem} 
\begin{proof}
    Due to the Lipschitz continuity of the least-squares projection $\bs{v}(\cdot)$, we have $\lVert \bs{v}(\bs{v}_{\mathrm{u}}) - \bs{v}(\bar{\bs{v}}_{\mathrm{u}})\rVert \leq L\lVert \bs{v}_{\mathrm{u}} - \bar{\bs{v}}_{\mathrm{u}}\rVert$. By applying the union bound to \cite{lederer2019uniform}, we obtain that, for every $\delta\in(0,1)$, there exist ${\beta}_1,\ldots,\beta_{d_v}>0$, such that 
    \[\lVert \bs{v}_u(\bs{z}) - \bar{\bs{v}}_u(\bs{z})\rVert \leq \sum_{i=1}^{d_v} \lvert \bs{v}_{u,i}(\bs{z}) - \bar{\bs{v}}_{u,i}(\bs{z})\rvert \leq \sum_{i=1}^{d_v} \sqrt{{\beta}_i}\sigma_{u,i}(\bs{z})\]
    holds for all states $\bs{z}$ with probability at least $1-\delta$. The result then follows by setting $\gamma=L\max_i\sqrt{\beta_i}$.
\end{proof}
Theorem \ref{theo:bounds} allows us to obtain an error bound for the prediction error that holds with high probability. In addition, as the position $\bs{x}_{k+1}$ is a linear function of the velocity, the bound can be applied straightforwardly for the entire state prediction.

\section{Evaluation and Discussion} \label{sec:evaluation}
We test our methods in several simulations to illustrate the theoretical results and evaluate the integrator's performance. Besides demonstrating the structure-preserving properties, we also investigate errors for multi-step predictions that are relevant in, for example, model-predictive control.   

For each system, training samples are drawn from 200 simulated trajectories starting with uniformly distributed random initial configurations and zero velocity. Each trajectory runs for 2 seconds with a time step of 100$\mu$s. Test samples are drawn from another 100 simulated trajectories with the same setup.
Each training sample consists of the full state input $\boldsymbol{z}_k$ and the velocity targets $\boldsymbol{v}_{k+1}$. We use an automatic relevance determination squared-exponential kernel of suitable dimensions for all experiments. Optimal hyperparameters are determined from 100 runs of maximizing the marginal likelihood with varying initial guesses for the hyperparameters.

\subsection{Energy conservation}
The desirable structure-preserving properties of the combination of variational integrators and Gaussian process regressions are demonstrated on a single and double pendulum with link masses $m=1$kg and lengths $l=1$m in maximal coordinates. For the simulations, we train a Gaussian process on a single recorded trajectory of two seconds starting at $\theta=\frac{\pi}{2}$ and $(\theta_1,\theta_2)=(\frac{\pi}{2},0)$, respectively.
from which the training samples are drawn. The resulting integrator is constructed without any prior dynamics knowledge according to \eqref{eq:var_int_gp_con}. For the single pendulum, the total energy error of the integrator \eqref{eq:var_int_gp_con} is compared to that of an explicit Euler integration of the real conservative system dynamics on the trained trajectory. Both integrators have a time step of $\Delta t=10$ms. The double pendulum is used for comparing the norm of the constraint drift for the explicit Euler method, and the integrator \eqref{eq:var_int_gp_con} with projection onto the constraints. The results are displayed in Fig. \ref{fig:energy}.

\begin{figure}
    \begin{tikzpicture}
        \node[inner sep=0pt] (russell) at (0,0)
            {\includegraphics[width=0.49\textwidth]{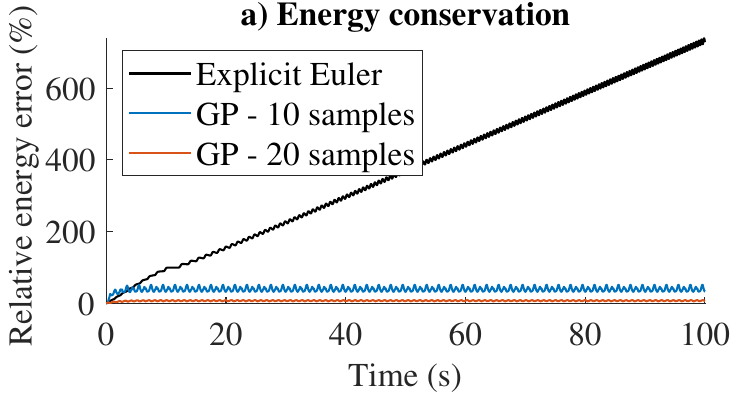}};
        \node[inner sep=0pt] (whitehead) at (7.6,0)
            {\includegraphics[width=.49\textwidth]{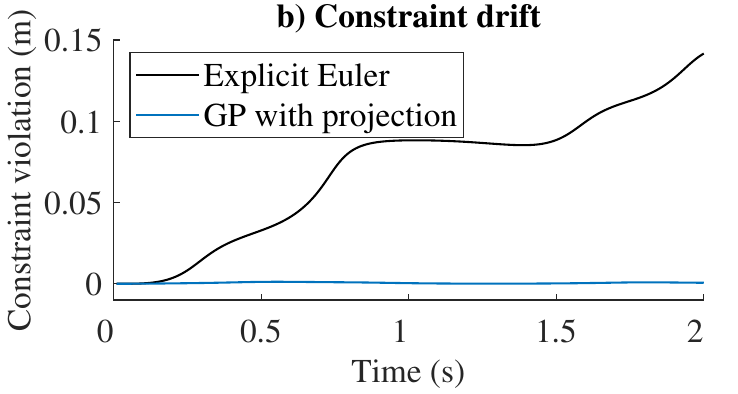}};
    \end{tikzpicture}
    \caption{Structure-preserving properties of the integration method even without prior dynamics model. a) Energy error for an explicit Euler integrator (black), integrator \eqref{eq:var_int_gp_con} with 10 training samples (blue), and integrator \eqref{eq:var_int_gp_con} with 20 training samples (red). b) Constraint drift for an explicit Euler integrator (black), and integrator \eqref{eq:var_int_gp_con} with projection (blue).}
    \label{fig:energy}
\end{figure}

The energy error for the explicit Euler integration in Fig. \ref{fig:energy} a) increases over time as is typical for explicit Runge-Kutta methods. In contrast, the learned dynamics models conserve the energy of the system even for longer time periods. It is observed that for fewer training samples, a certain amount of energy error is reached before conservation occurs. One possible explanation is that a trained area of the state space has to be reached for energy conservation. The small oscillations in the energy error are typical for variational integrators (see \cite{marsden_discrete_2001}). In Fig. \ref{fig:energy} b), the constraint drift of the explicit Euler method for the double pendulum becomes clearly visible, which occurs because constraints are formulated on an acceleration level, and not a position level. In contrast, no drift occurs for the trained integrator \eqref{eq:var_int_gp_con}.

\subsection{Prediction Performance}
\begin{figure}[htp!]
    \begin{tikzpicture}
        \coordinate (Lw1) at (-2.0,5.0);
		\draw[fill=black!5,rotate=100] ($(Lw1)+(0,0.0)$) rectangle ++(2,-0.2);
		\draw[fill=black] ($(Lw1)+(-0.25,2.0)$) circle (0.08);
		\node[font=\bfseries] at ($(Lw1)+(0.0,2.8)$) {Pendulum};
        \node[font=\bfseries] at ($(Lw1)+(0.8,2.0)$) {$c \in [0,1]$};

        \coordinate (Lw2) at (1.5,5.0);
        \draw[fill=black,rotate=0] ($(Lw2)+(-1.5,-1.55)$) rectangle ++(3,-0.1);
		\draw[fill=black!5,rotate=100] ($(Lw2)+(-1.6,0.3)$) rectangle ++(2,-0.2);
        \draw[fill=black!5,rotate=0] ($(Lw2)+(-1.0,-1.5)$) rectangle ++(2,-0.2);
		\draw[fill=black] ($(Lw2)+(0.06,-1.5)$) circle (0.08);
		\node[font=\bfseries] at ($(Lw2)+(0.0,2.8)$) {Cartpole};
        \node[font=\bfseries] at ($(Lw2)+(1.0,-1.1)$) {$c_2 \in [0,1]$};
        \node[font=\bfseries] at ($(Lw2)+(-2.0,-1.2)$) {$c_1 \in [0,0.5]$};

        \coordinate (Lw3) at (4.5,5.0);
		\draw[fill=black!5,rotate=100] ($(Lw3)+(0,0.0)$) rectangle ++(2,-0.2);
		\draw[fill=black!5,rotate=120] ($(Lw3)+(-2,0)$) rectangle ++(2,-0.2);
		\draw[fill=black] ($(Lw3)+(-0.25,2.0)$) circle (0.08);
		\draw[fill=black] ($(Lw3)+(0.11,0.05)$) circle (0.08);
		\node[font=\bfseries] at ($(Lw3)+(0.0,2.8)$) {Double Pendulum};
        \node[font=\bfseries] at ($(Lw3)+(0.8,2.0)$) {$c_1 \in [0,2]$};
        \node[font=\bfseries] at ($(Lw3)+(1.2,0.4)$) {$c_2 \in [0,0.5]$};

        \coordinate (Lw4) at (8.0,5.0);
		\draw[fill=black!5,rotate=100] ($(Lw4)+(0,0.0)$) rectangle ++(2,-0.2);
		\draw[fill=black!5,rotate=120] ($(Lw4)+(-2,0)$) rectangle ++(2,-0.2);
        \draw[fill=black!5,rotate=120] ($(Lw4)+(-0.15,-0.7)$) rectangle ++(2,-0.2);
		\draw[fill=black!5,rotate=100] ($(Lw4)+(-1.85,-0.7)$) rectangle ++(2,-0.2);
		\draw[fill=black] ($(Lw4)+(-0.25,2.0)$) circle (0.08);
		\draw[fill=black] ($(Lw4)+(0.11,0.05)$) circle (0.08);
        \draw[fill=black] ($(Lw4)+(0.75,0.27)$) circle (0.08);
        \draw[fill=black] ($(Lw4)+(1.1,-1.7)$) circle (0.08);
		\node[font=\bfseries] at ($(Lw4)+(0.3,2.8)$) {Fourbar Segment};
        \node[font=\bfseries] at ($(Lw4)+(0.8,2.0)$) {$c_1 \in [0,2]$};
        \node[font=\bfseries] at ($(Lw4)+(1.1,0.6)$) {$c_2=c_3 \in [0,0.5]$};
        \node[font=\bfseries] at ($(Lw4)+(0.1,-1.5)$) {$c_4 = 0$};

        \node[inner sep=0pt] (russell) at (0,0)
            {\includegraphics[width=.49\textwidth]{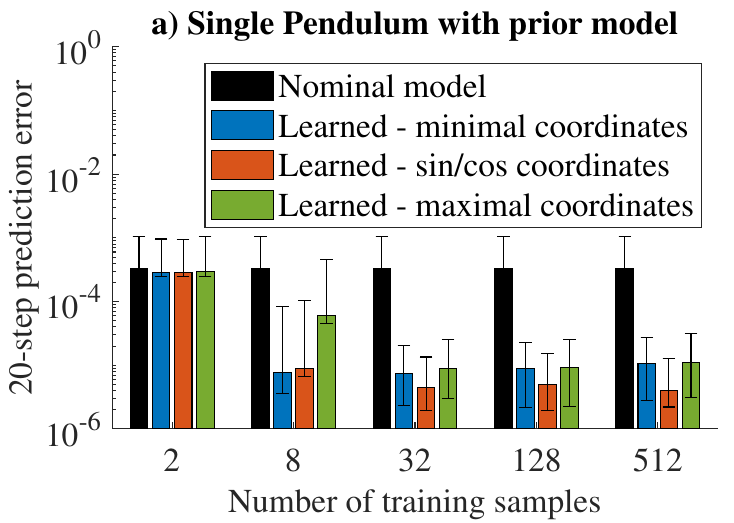}};
        \node[inner sep=0pt] (whitehead) at (7.6,0)
            {\includegraphics[width=.49\textwidth]{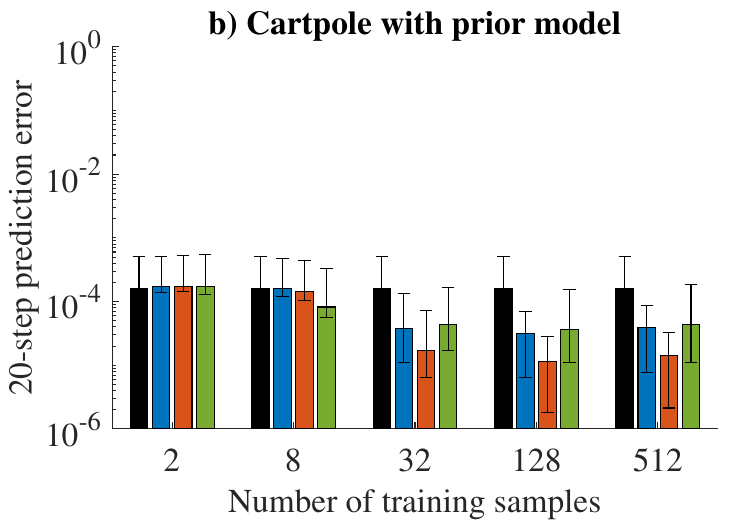}};
        \node[inner sep=0pt] (russell) at (0,-5.7)
            {\includegraphics[width=.49\textwidth]{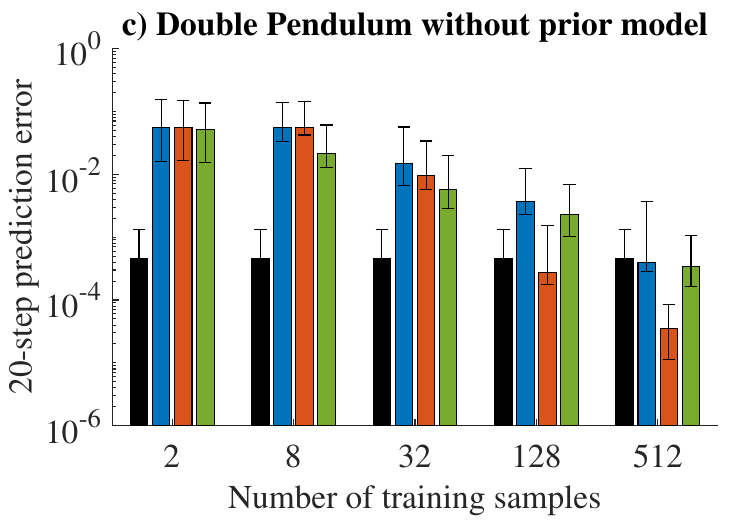}};
        \node[inner sep=0pt] (whitehead) at (7.6,-5.7)
            {\includegraphics[width=.49\textwidth]{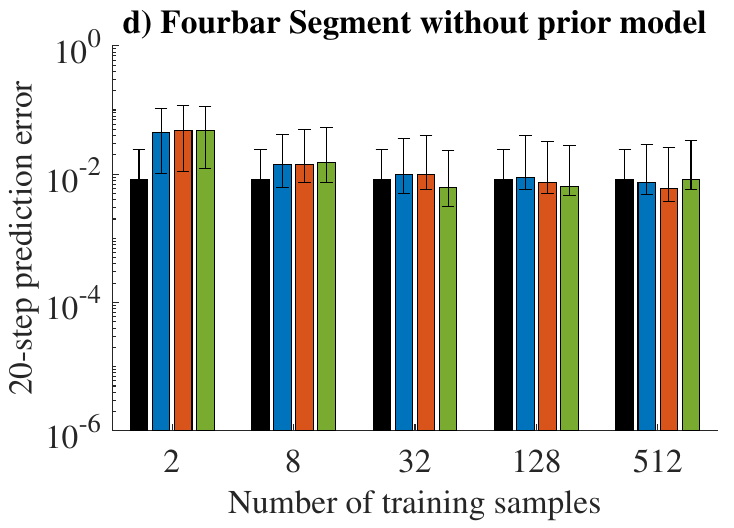}};
    \end{tikzpicture}
    \caption{Median mean-square position error and 10 and 90 percentiles for 20-step predictions of perturbed systems (top, with friction coefficients) trained on varying sample numbers and integrated with different integrators: variational integrator of the nominal dynamics model without learning (black), integrator \eqref{eq:var_int_gp} trained in minimal coordinates (blue), integrator \eqref{eq:var_int_gp} trained in sin/cos coordinates (red), integrator \eqref{eq:var_int_gp_con} trained in maximal coordinates (green). a) Error for a single pendulum with known prior mean dynamics. b) Error for a cartpole with known prior mean dynamics, c) Error for a double pendulum without prior mean dynamics, d) Error for a fourbar segment without prior mean dynamics.}
    \label{fig:exp_meandyn}
\end{figure}
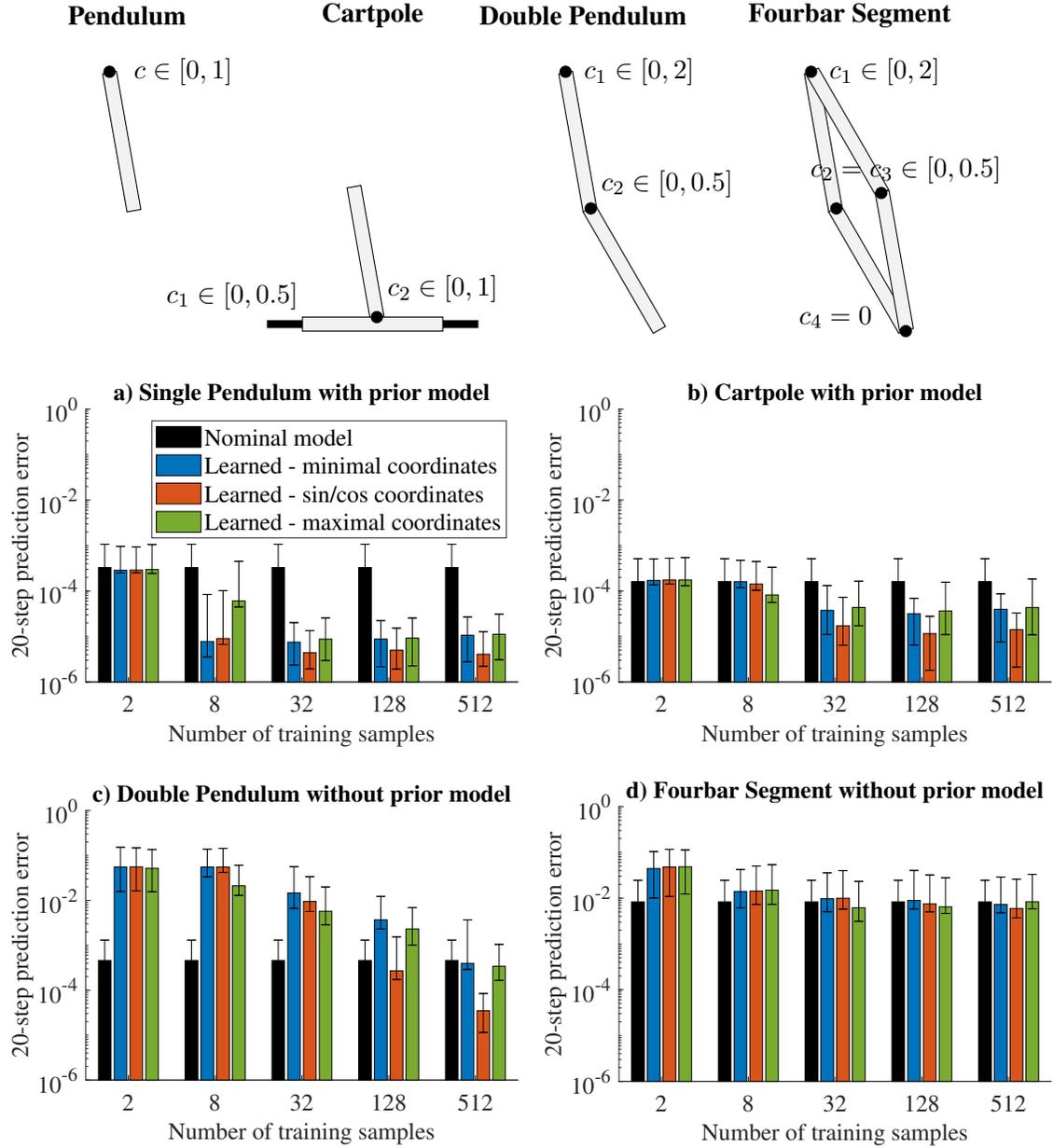
The performance of the proposed integration methods is evaluated by comparing it to integrating a nominal dynamics model and calculating the deviation from the ground-truth simulation of a perturbed system. The ground-truth simulation with a symplectic Euler method uses a time step of $\Delta t=0.1$ms, whereas the other integrators use a time step of $\Delta t=10$ms. As test systems we selected a pendulum, a cartpole, a double pendulum, and a fourbar segment (closed kinematic chain). The systems are shown in Fig. \ref{fig:exp_meandyn}. All links of the systems have a mass $m=1$kg, length $l=1$m, and inertia $J = \frac{1}{12}ml^2$.

For the Gaussian-process-based integrators, three coordinate descriptions are compared: minimal (joint) coordinates, a complex-number representation, i.e., $[\sin(\theta) ~ \cos(\theta)]\hT$ instead of joint angles $\theta$, and maximal coordinates. We perform 20-step ahead predictions in all cases. 

\subsubsection{Prediction with Prior Mean} \label{subsec:varint_residuals}
For the evaluation with a prior mean, i.e., a nominal dynamics model, a single pendulum and cartpole are used. The masses and inertias of the links of the nominal system models are distorted by random uniformly distributed ([0.9,1.1]) perturbation factors, for example $m'=0.9m$. Additionally, viscous friction was added for the joints (see uniformly distributed coefficients $c$ in Fig. \ref{fig:exp_meandyn}). Each observed state is disturbed with additive zero mean Gaussian noise ($\sigma = 10^{-3}$) for training. The results are shown in Fig. \ref{fig:exp_meandyn} a) and b).

As expected, the results in Figure \ref{fig:exp_meandyn} a) and b) indicate that including a prior mean model yields reasonable results for the proposed integrators despite noise and distortion. The uncertainty and error decreases (note the log scale) with additional training samples, highlighting the advantage of including Gaussian process regression for learning residual dynamics compared to just using the nominal model. Minimal and maximal coordinates achieve similar results, justifying the use of the projection for systems with explicit constraints, while the sin/cos parameterization outperforms the other two coordinate descriptions. It appears that no further improvement is achieved once a certain number of training sample are used. One possible explanation for this result could be that the optimization of hyperparameters gets stuck in local minima.

\subsubsection{Prediction without Prior Mean} \label{subsec:varint_learned}
We also compare the performance of the learning-augmented integrator without a prior dynamics model, i.e., without a prior mean, on a double pendulum and a fourbar segment. The Gaussian process regression uses the constant mean from the training data as its mean function, a commonly used strategy in situations without a prior mean. The observation noise and friction as well as the perturbation factors for masses and inertias remain the same as in Sec. \ref{subsec:varint_residuals}. The results are shown in Fig. \ref{fig:exp_meandyn} c) and d).

In Fig. \ref{fig:exp_meandyn}, unlike the predictions with a prior mean, the errors and uncertainties reached are much higher for few training samples compared to integrating the nominal dynamics. As the number of training samples increases, the accuracy improves until it is close to the prediction accuracy of the integrators with a prior nominal model and partially surpasses the pure nominal dynamics. As before, minimal and maximal coordinates result in similar behavior. In the case of the pendulum, the sin/cos parameterization outperforms both. We only evaluate up to 512 samples due to the long computation times for optimizing hyperparameters, which would also be undesirable in an experimental implementation.

\section{Conclusions}\label{sec:conclusions}
We have presented an approach for structure-preserving learning of mechanical systems by combining variational integrators and Gaussian process regression that accurately model and integrate such systems. Known kinematic constraints can be treated with this approach as well. For correct predictions with the Gaussian process regression, the developed method is a variational integrator with corresponding symplectic properties. The evaluation of the proposed methods on a variety of mechanisms in simulation shows satisfactory results both for unconstrained and constrained parameterizations. The data-efficiency of Gaussian process regression and accurate predictions of variational integrators even for large time steps make the proposed method interesting for real-time control applications.

\acks{We thank Armin Lederer, Samuel Tesfazgi, and Petar Bevanda for their help in preparing this manuscript. This work
was supported by the European Union's Horizon 2020 research and innovation programme
under grant agreement no. 871295 "SeaClear".}

\bibliography{l4dc}

\end{document}